%% file: main.tex
\documentclass[twoside,lettersize,journal]{IEEEtran}
\usepackage{array}
\usepackage{textcomp}
\usepackage{stfloats}
\usepackage{url}
\usepackage{verbatim}
\def\BibTeX{{\rm B\kern-.05em{\sc i\kern-.025em b}\kern-.08em
    T\kern-.1667em\lower.7ex\hbox{E}\kern-.125emX}}
\usepackage{balance}

\input{macros}
\begin{document}

\title{GamePlan: Game-Theoretic Multi-Agent Planning with Human Drivers at Intersections, Roundabouts, and Merging}

\author{Rohan Chandra and Dinesh~Manocha
\thanks{This work was supported in part by ARO Grants W911NF1910069 and W911NF1910315, Semiconductor Research Corporation (SRC), and Intel. All authors are with the the Department
of Computer Science, University of Maryland, College Park,
MD, 20742, USA. e-mail: rchandr1@umd.edu. Supplementary material provided in the \href{https://arxiv.org/pdf/2109.01896.pdf}{full arXiv Report} (Pg. $9$-$11$).}
}

\maketitle

\input{Sections/0-Abstract}
\input{Sections/1-intro}
\input{Sections/2-related}
\input{Sections/3-background}

\input{Sections/4-gameplan}

\input{Sections/Theory}
\input{Sections/6-experiments}
\input{Sections/7-conclusion}

{\footnotesize \bibliography{refs}}
\bibliographystyle{plain}

\end{document}

%% file: macros.tex
\newcommand{\model}{\textsc{GamePlan}}

\makeatletter
\newcommand\footnoteref[1]{\protected@xdef\@thefnmark{\ref{#1}}\@footnotemark}
\makeatother

\newcommand{\shorteq}{%
  \settowidth{\@tempdima}{-}
  \resizebox{\@tempdima}{\height}{=}%
}

\usepackage{amssymb,fge}

\usepackage[table]{xcolor}
\definecolor{sg}{HTML}{00ff7f}
\definecolor{lb}{HTML}{b0f5ef}
\definecolor{lg}{HTML}{9bfaa8}
\DeclareRobustCommand{\hllg}[1]{#1}

\DeclareRobustCommand{\hllb}[1]{#1}
\usepackage{amsmath, amssymb}
\usepackage{pifont}
\newcommand{\cmark}{\ding{51}}%
\newcommand{\xmark}{\ding{55}}%
\usepackage{subcaption}
\usepackage[utf8]{inputenc}
\usepackage[english]{babel}
\usepackage[table]{xcolor}
\usepackage[linesnumbered,ruled,vlined]{algorithm2e}
\usepackage[font=small,belowskip=-1.5pt]{caption} 
\usepackage{anyfontsize}

\usepackage{array}
\usepackage{graphicx}
\usepackage{amsfonts}
\usepackage[11pt]{moresize}
\usepackage{bm}
\usepackage{soul}
\usepackage{hhline}
\usepackage{multirow, makecell}
\usepackage{float}
\usepackage{booktabs,wrapfig}

\usepackage{amsthm}
\usepackage{color}
\usepackage{transparent}
\usepackage{url}
\usepackage{footmisc}
\usepackage{setspace}
\usepackage[colorlinks=true, citecolor=magenta]{hyperref}
\usepackage{mathtools}
\theoremstyle{plain}

\newtheorem{theorem}{Theorem}[section]

\newtheorem{definition}{Definition}[section]

\newtheorem{problem}{Problem}[section]

%% file: Sections/0-Abstract.tex
\begin{abstract}


We present a new method for multi-agent planning involving human drivers and autonomous vehicles (AVs) in unsignaled intersections, roundabouts, and during merging. In multi-agent planning, the main challenge is to predict the actions of other agents, especially human drivers, as their intentions are hidden from other agents. Our algorithm uses game theory to develop a new auction, called \model, that directly determines the optimal action for each agent based on their driving style (which is observable via commonly available sensors). \model~assigns a higher priority to more aggressive or impatient drivers and a lower priority to more conservative or patient drivers; we theoretically prove that such an approach is game-theoretically optimal prevents collisions and deadlocks. We compare our approach with prior state-of-the-art auction techniques including economic auctions, time-based auctions (first-in first-out), and random bidding and show that each of these methods result in collisions among agents when taking into account driver behavior. We compare with methods based on DRL, deep learning, and game theory and present our benefits over these approaches. Finally, we show that our approach can be implemented in the real-world with human drivers.


\end{abstract}


%% file: Sections/1-intro.tex
\section{Introduction}
\label{sec: introduction}
Navigating unsignaled intersections, roundabouts, and merging scenarios is a challenging problem for autonomous vehicles (AVs). $40\%$ of all crashes, $50\%$ of serious collisions, and $20\%$ of fatalities occur at intersections~\cite{grembek2018introducing}. Planning in such scenarios is difficult since agents do not know the objective functions of other agents and therefore may not be able to coordinate their actions. For instance, consider a commonly occurring scenario consisting of two human drivers (green and blue agent) arriving at a four-way, unsignaled intersection at approximately the same time (Figure~\ref{fig: cover}). Both drivers may decide to take the initiative and move first, at the same time, potentially resulting in a collision (Figure~\ref{fig: nogameplan}).

The current methods for navigating unsignaled intersections, roundabouts, or merging scenarios use deep reinforcement learning~\cite{liu2020decision, isele2018navigating, capasso2021end}, game theory~\cite{li2020game, tian2020game}, recurrent neural networks~\cite{roh2020multimodal}, and auctions~\cite{carlino2013auction, sayin2018information, rey2021online}. While these methods successfully perform multi-agent planning, they however do not guarantee collision- and deadlock-free navigation. However, certain auctions that are game-theoretically optimal can navigate unsignaled intersections, roundabouts, or merging scenarios~\cite{carlino2013auction, sayin2018information, rey2021online} that are collision- and deadlock-free. In a traffic auction, each driver (bidder) bids whether to move or wait based on a specific bidding strategy. A centralized auction program collects the bids and allocates each agent a position (goods) that determines when they should move. Such a determination is meant to be acceptable to each agent. Furthermore, the agents make their bid without the knowledge of the objective functions of the other agents. The most common auction observed in most regions is the first-in-first-out (FIFO)~\cite{gt6} in which the bidding strategy is based on the agents' arrival times; but more sophisticated and demonstrably better economic auctions utilizing monetary-based bidding strategies have also been proposed.
\input{img/sim_results/sim_results}

Game-theoretically optimal auctions are hard to design. Without these optimality guarantees, agents may be incentivized to change their bids to increase their utility; but changing bids mid-navigation results in collisions and deadlocks. The state-of-the-art optimal auctions currently available are economic auctions~\cite{carlino2013auction, sayin2018information, rey2021online} wherein agents are assigned wallets with monetary budgets. The issue with these auctions is that they are biased towards wealthier agents~\cite{gt6} and disregard social preferences of human drivers. For instance, consider impatient agents wanting to move first that have smaller budgets. Such agents despite being outbid by more conservative agents with higher budgets may be incentivized to falsify their bids--unbeknownst to other agents--in order to move first. This may result in a collision with an agent that was originally chosen to move before the aggressive agent. Therefore, new auction designs for navigation must be developed in which the bidding, allocation, and payment rules would be based, not only on the monetary budget of an agent but on other aspects as well. In this work, we design a new robotics-based auction framework for navigating unsignaled intersections, roundabouts, and merging scenarios with multiple human drivers based on their driving behavior.

\noindent\textbf{Main contributions: }We propose the following contributions. 
\begin{enumerate}
    \item We propose \model, a robotics-based game-theoretically optimal auction, in which vehicle trajectories and velocities are used to directly estimate a driver's aggressiveness or impatience. \model~uses driver behavior to define the bidding strategy. The output of the auction is an optimal turn-based ordering in which more aggressive or impatient drivers move first.
    \item We show that \model~outperforms other multi-agent planning methods with significantly reduced collisions and deadlocks in unsignaled intersections, roundabouts, and merging scenarios. 
\end{enumerate}

\noindent Compared to methods based on deep reinforcement learning, game theory, and recurrent neural networks, \model~achieves on average, at least $10-20\%$ decrease in the rate of collisions and deadlocks. We also compare \model~with economic auctions, random bidding strategies (for \textit{e.g.} using a driver's height or weight as the bidding strategy), and the FIFO principle, and outperforms the current state-of-the-art economic auctions by approximately $3.5\%$. Finally, we demonstrate \model~works successfully in two real-world merging scenarios involving real human drivers.

%% file: img/sim_results/sim_results.tex
\begin{figure}[t]
\centering
   \begin{subfigure}[h]{0.492\columnwidth}
    \includegraphics[width=\textwidth]{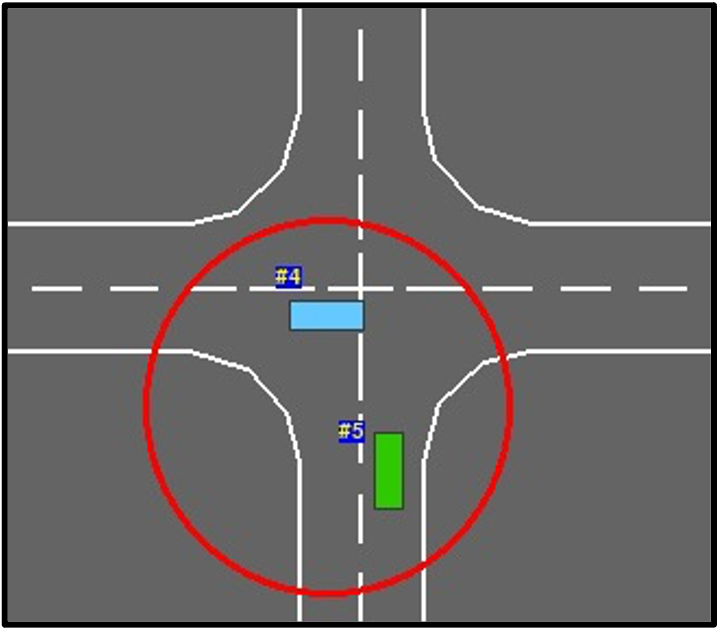}
    \caption{\textbf{Without \textsc{GamePlan}:} Imminent collision occurs when agents enter an unsignaled intersection.}
    \label{fig: nogameplan}
  \end{subfigure}
 \begin{subfigure}[h]{0.492\columnwidth}
    \includegraphics[width=\textwidth]{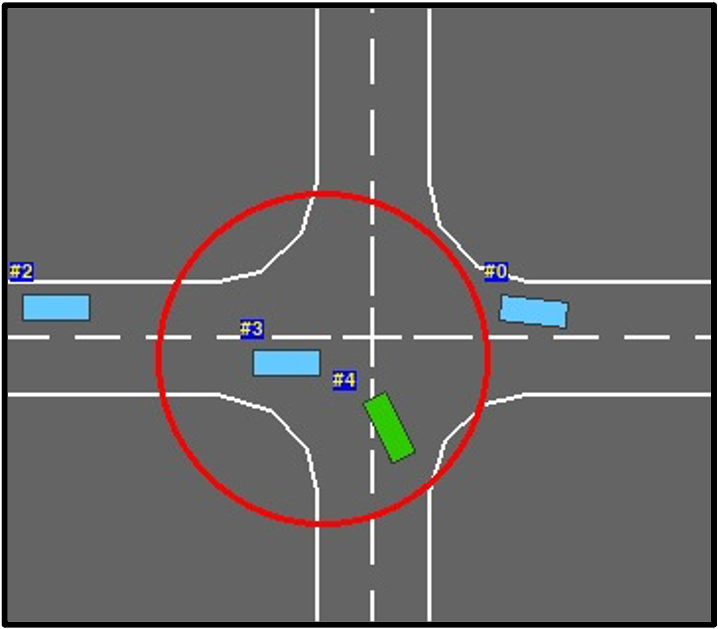}
    \caption{\textbf{With \model:} the green agent allows the more aggressive blue agent to pass first.}
    \label{fig: gameplan}
  \end{subfigure}
 
  %
%
  %
\caption{\textbf{\model:} We present the first approach of using driver behavior to compute a collision- and deadlock-free turn-based ordering for non-communicating agents at unsignaled intersections, roundabouts, and during merging. Figure~\ref{fig: nogameplan} shows that two impatient agents arriving at the intersection at the same time are more likely to collide due to their implicit nature of moving first. With \model, we compute a turn-based ordering in which more aggressive agents are allowed to pass first. Thus, in Figure~\ref{fig: gameplan}, the more aggressive of the two agents (blue agent) crosses the intersection followed by the green agent, thereby preventing a collision.}
  \label{fig: cover}
  \vspace{-5pt}
\end{figure}

%% file: Sections/2-related.tex
\section{Prior Work}
\label{sec: related_work}
In Table~\ref{tab: bidding_rules}, we compare our approach with the current state-of-the-art in navigating unsignaled intersections, roundabouts, and merging scenarios on the basis of optimality guarantees, multi-agent versus single-agent planning (MAP), action space (AS), incentive compatibility (IC), and real-world applicability. 

\subsubsection{Deep reinforcement learning (DRL)}

DRL-based methods~\cite{capasso2021end, isele2018navigating, Kai2020AMR, liu2020decision, bgap} learn a navigation policy using the notion of expected reward received by an agent from taking a particular action in a particular state. This policy is learned from trajectories obtained via traffic simulators using Q-learning~\cite{dql} and is very hard as well as expensive to train. In practice, DRL-based planning methods often do not generalize well to different environments and it is hard to provide any guarantees. Furthermore, these methods discussed so far are intended for single-agent navigation. However, Capasso et al.~\cite{capasso2021end} use additional signals such as traffic signs (stop, yield, none) to regulate the movement and actions of multiple agents. In terms of real world applications, Kai et al.~\cite{Kai2020AMR} learn a unified policy for multiple tasks and also demonstrate their approach on a real robot.

\subsubsection{Game theory}

Game-theoretic approaches~\cite{li2020game,tian2020game}, by their very nature, perform multi-agent planning. However, these methods restrict the actions and objectives of active agents. For example, Li et al.~\cite{li2020game} formulate traffic intersection planning as a stackelberg leader-follower game in which one agent is assumed to act first (leader) and the other agent (follower) will react accordingly. The leader is decided by a first-in first-out (FIFO) principle. Tian et al.~\cite{tian2020game} use a recursive $k-$level game-theoretic approach in which complex strategies for agents at each level are derived from previous levels. However, all agents except the ego-agent at the first level are assumed to be static.

\input{Tables/bidding}

\subsubsection{Recurrent neural networks (RNNs)}

Deep learning-based methods~\cite{roh2020multimodal} train a recurrent neural network for trajectory prediction and are also susceptible to complex environments or different behavior of drivers. In contrast to all of the above, \model~do not require an objective function; instead, successful application of this approach requires minimizing a loss function that depends on the distribution of the data.

\subsubsection{Auctions}

A very basic principle to regulate traffic at intersections is the FIFO principle which basically states that agents move in the order in which they arrive at the intersection. The main concern with such an approach is that aggressive or impatient drivers are incentivised to break form and move out of turn. Additionally, in situations when drivers arrive at an intersection at the same time, oftentimes deadlocks occur as a result of ensuing confusion among the drivers. Therefore, FIFO is not guaranteed to be optimal. Buckman et al.~\cite{gt6} integrate a driver behavior model~\cite{schwarting2019social} within the basic FIFO framework to incorporate human social preference to address some of the above limitations of FIFO. But the model does not estimate the social preferences in real time; instead, it chooses a fixed preference parameter for each agent.

Vasirani and Ossowski~\cite{vasirani2012market} proposed a combinatorial auction for assigning turns at intersections. Censi et al.~\cite{censi2019today} introduced a karma-based auction and Lin and Jabari~\cite{lin2021pay} proposed a mechanism for pricing intersection priority based on transferable utility games. However, these auctions are not incentive-compatible. Incentive-compatible auctions such as Sayin et al.~\cite{sayin2018information} propose a mechanism in which agents are assigned turns based on their distance from the intersection and the number of passengers in the vehicle. Carlino et al.~\cite{carlino2013auction} and Rey et al.~\cite{rey2021online} propose a similar mechanism but use a monetary-based bidding strategy. These methods are, however, biased towards wealthier agents disregarding human preference, limiting their use in the real-world applications. 

%% file: Tables/bidding.tex
\begin{table}
\centering
\caption{\textbf{Summary of prior work:} We list methods for navigating unsignaled intersections, roundabouts, and merging based on multi-agent planning (MAP), action space (AS), and incentive compatibility (IC). \cmark$^{*}$ corresponding to a method indicates that optimality does not hold for human drivers with varying social preferences.}
\resizebox{\columnwidth}{!}{
\begin{tabular}{lrccccc}
\toprule
Approach & Methods & Optimality & MAP & AS & IC & Real world\\
\midrule
\multirow{4}{*}{\small DRL}& Capasso et al.~\cite{capasso2021end} &\cellcolor{red!25} \xmark&\cellcolor{green!25} \cmark &C & -& \cellcolor{red!25} \xmark\\
& Isele et al.~\cite{isele2018navigating} &\cellcolor{red!25} \xmark&\cellcolor{red!25} \xmark&D & -& \cellcolor{red!25} \xmark\\
& Kai et al.~\cite{Kai2020AMR} &\cellcolor{red!25} \xmark&\cellcolor{red!25} \xmark& D & -& \cellcolor{green!25} \cmark\\
& Liu et al.~\cite{liu2020decision} &\cellcolor{red!25} \xmark&\cellcolor{red!25} \xmark&D & -& \cellcolor{red!25} \xmark\\
\cmidrule{2-7}

\multirow{2}{*}{\small Game Theory}& Li et al.~\cite{li2020game} &\cellcolor{red!25} \xmark&\cellcolor{green!25} \cmark&C & -& \cellcolor{red!25} \xmark\\
& Tian et al.~\cite{tian2020game} &\cellcolor{red!25} \xmark&\cellcolor{green!25} \cmark&C & -& \cellcolor{red!25} \xmark\\
\cmidrule{2-7}

{RNN} & Roh et al.~\cite{roh2020multimodal} &\cellcolor{red!25} \xmark&\cellcolor{green!25} \cmark&C & -& \cellcolor{red!25} \xmark\\
\cmidrule{2-7}


\multirow{8}{*}{\small Auctions}& FIFO &\cellcolor{red!25} \xmark&\cellcolor{green!25} \cmark&D & -& \cellcolor{red!25} \xmark\\
& Buckman et al.~\cite{gt6} &\cellcolor{red!25} \xmark&\cellcolor{green!25} \cmark&D & -& \cellcolor{red!25} \xmark\\
& Vasirani and Ossowski~\cite{vasirani2012market}&\cellcolor{red!25} \xmark&\cellcolor{green!25} \cmark&D & \cellcolor{red!25} \xmark& \cellcolor{red!25} \xmark\\
& Censi et al.~\cite{censi2019today}&\cellcolor{red!25} \xmark&\cellcolor{green!25} \cmark&D & \cellcolor{red!25} \xmark& \cellcolor{red!25} \xmark\\
& Lin and Jabari~\cite{lin2021pay}&\cellcolor{red!25} \xmark&\cellcolor{green!25} \cmark&D & \cellcolor{red!25} \xmark& \cellcolor{red!25} \xmark\\
& Carlino et al.~\cite{carlino2013auction} &\cellcolor{green!25} \cmark$^{*}$&\cellcolor{green!25} \cmark&D & \cellcolor{green!25} \cmark& \cellcolor{red!25} \xmark\\
& Rey et al.~\cite{rey2021online}&\cellcolor{green!25} \cmark$^{*}$&\cellcolor{green!25} \cmark&D & \cellcolor{green!25} \cmark& \cellcolor{red!25} \xmark\\
& Sayin et al.~\cite{sayin2018information} &\cellcolor{green!25} \cmark$^{*}$&\cellcolor{green!25} \cmark&D & \cellcolor{green!25} \cmark& \cellcolor{red!25} \xmark\\
\cmidrule{2-7}
& \textbf{This work} &\cellcolor{green!25} \cmark&\cellcolor{green!25} \cmark&D & \cellcolor{green!25} \cmark & \cellcolor{green!25} \cmark\\

\bottomrule
\end{tabular}
}
\label{tab: bidding_rules}
\vspace{-10pt}
\end{table}

%% file: Sections/3-background.tex
\section{Background and Problem Formulation}
In this section, we begin by defining several key terms and stating the problem statement and assumptions made in our approach~(Section~\ref{subsec: problem_formulation}). We briefly summarize the CMetric algorithm~\cite{cmetric} for behavior modeling and prediction that we use in \model~in Section~\ref{subsec: cmetric_background}. Finally, we provide an overview on sponsored search auctions~\cite{roughgarden2016twenty} in Section~\ref{subsec: SSA_background} which forms the basis of the auction used in \model.

\subsection{Problem Formulation}
\label{subsec: problem_formulation}

We frame navigating unsignaled intersections, roundabouts, and merging scenarios as a multi-agent game-theoretic non-sequential decision-making problem. We consider three unsignaled and uncontrolled traffic scenarios or environments -- intersections, merging, and roundabouts with multiple non-communicating human drivers and AVs. In each of the three scenarios, we refer to a traffic-agent as ``active'' if they participate in the decision-making process for navigating the scenario. For example, the lead traffic-agents in each service lane of a four-way intersection are active vehicles, while agents waiting in line behind their respective lead vehicle are non-active agents. We denote the active agents as a set $\mathcal{A} = \{a_1, a_2, \ldots a_n \}$ where each $a_i$ participates in the \model~auction. The action space for each active agent is a discrete set consisting of finite acceleration values. The state space consists of the position $x_t \in \mathcal{X}$ and velocity $v_t$ of each agent measured in a global coordinate frame (\textit{e.g.} GPS coordinates or $3$D points obtained using GPS and lidars).

\noindent\textbf{Problem Statement: } To prevent collisions and deadlocks, it is necessary to determine the order in which agents take turns to navigate the traffic scenario~\cite{carlino2013auction, dresner2008multiagent, lin2021pay, sayin2018information, schepperle2007agent, vasirani2012market}. Thus, the goal of this paper is to compute an optimal turn-based ordering which is defined as follows,

\begin{definition}
\textbf{Turn-based Orderings ($\bm{\sigma}$): } A turn-based ordering is a permutation $\bm{\sigma}: \mathcal{A} \rightarrow [1,n]$ over the set $\mathcal{A}$. For any $i,j \in [1,n]$, $\bm{\sigma}(a_i) = j$, equivalently $\bm{\sigma}_i = j$, indicates that $a_i$ with driving behavior $\bm{\zeta}_i$ will move on the $j^\textrm{th}$ turn.
\label{def: turn_based_ordering}
\end{definition}

\noindent There may be multiple turn-based orderings that may apply in a given scenario. However, many of them could be sub-optimal and may result in collisions and deadlocks. An optimal turn-based ordering prevents collisions and resolves deadlock conflicts.

\begin{definition}
\textbf{Optimal Turn-based Ordering ($\bm{\sigma}_{\textsc{opt}}$): } A turn-based ordering, $\bm{\sigma}$, over a given set $\mathcal{A}$ is game-theoretically optimal if it is incentive compatible, welfare maximizing, and can be computed in polynomial time\footnote{These terms are defined in Section~\ref{subsec: optimality}}.
\label{def: optimal_turn_based_ordering}
\end{definition}

\noindent We are now ready to formally state our goal in this paper.

\begin{problem}
In a traffic scenario (which could be an intersection, roundabout, or merging) with $n$ active road-agents, given the behavior profile $\bm{\zeta}_i$ of each agent, the goal is to compute the optimal turn-based ordering $\bm{\sigma}_\textsc{opt}  = \bm{\sigma}_1\bm{\sigma}_2\ldots\bm{\sigma}_n$.
\label{problem: prob}
\end{problem}

\noindent\textbf{Assumptions: }We assume agents (AVs and human drivers) are rational and strive to maximize their own utility which can be independent of other agents. Further, agents are \textit{non-ideal} (in contrast to ideal agents in~\cite{cleac2020lucidgames, cleac2019algames, wang2020game}) in that they do not have access to the actions, objectives, or utility functions of other agents. We make no assumptions on the motion or dynamics of traffic-agents. The state-space is fully observable by all agents (active and non-active). We assume the availability of a behavior modeling and prediction algorithm such as CMetric~\cite{cmetric} that models driving behavior as being either aggressive or conservative. This information is provided to an auction program. We assume that agents navigate through the environment one at a time and do not use traffic signals, employ right-of-way rules, or communicate with other agents.

\subsection{Modeling human driver behavior}
\label{subsec: cmetric_background}

Each active agent is characterized by a behavior profile, $\bm{\zeta}_i$, which can be obtained via recent behavior modeling algorithms such as the SVO~\cite{schwarting2019social} and CMetric~\cite{cmetric}. In this work, we use CMetric to quantify the aggressiveness or conservativeness of a traffic-agent. This model provides an objective measure of aggressiveness based on driving maneuvers such as overspeeding, overtaking and so on. We briefly summarize the CMetric algorithm. To determine if an agent is aggressive or conservative, the algorithm begins by reading the trajectories of the ego-agent and surrounding vehicles via cameras or lidars during any time-period $\Delta t$. The trajectory of an agent is represented by 

\[\Xi_{\Delta t} = \{x_t \ | \ t = t_0, t_1, \ldots, t_0 + \Delta t \}.\]

\noindent In CMetric, these trajectories are represented via weighted undirected graphs $\mathcal{G} = (\mathcal{V}, \mathcal{E})$ in which the vertices denote the positions for each agent and the edges correspond to the distances between agents. The algorithm proceeds by using these graphs to model the likelihood and intensity of driving behavior indicators like overspeeding, overtaking, sudden ``zig-zagging'' and lane-changes via centrality functions~\cite{rodrigues2019network} represented by $\Phi: \mathcal{G}\rightarrow \mathbb{R}$. These behavior indicators determine whether an agent is aggressive or conservative. The behavior profile for $a_i$ is denoted by $\bm{\zeta}_i$ and is computed as,

\begin{equation}
    \bm{\zeta}_i(\Xi_{\Delta t}) = \Phi^i(\mathcal{G})[t],
    \label{eq: cmetric}
\end{equation}

\noindent where $\mathcal{G}(\mathcal{V}, \mathcal{E})$ is constructed using $\Xi_{\Delta t}$. The original definition of centrality, however, does not take into account temporal memory since the the centrality value changes with time. In order to model driver behavior during a time period $\Delta T$, we must keep track of all the neighbors the vehicle has interacted with during that period. $\Phi(\mathcal{G})[t]$ is defined as follows,

\begin{definition}
\textbf{Temporal memory for} $\bm{\Phi(\mathcal{G})}$: In a connected traffic-graph $\mathcal{G}$ at time $t$ with associated adjacency matrix $A_t$, let $\mathcal{N}_i(t) = \{ v_j \in \mathcal{V}(t), \ A_t(i,j) \neq 0, \nu_j \leq \nu_i\}$ denote the set of vehicles in the neighborhood of the $i^\textrm{th}$ vehicle with radius $\mu$, then the discrete degree centrality function of the $i^\textrm{th}$ vehicle at time $t$ is defined as,
\begin{equation}
    \begin{aligned}
    \Phi^i[t] = \Phi^i \left( \left\{ v_j \in \mathcal{N}_i(t) \right\} \right) + 
    \Phi^i[t-1] &\\
    \textrm{such that} \ (v_i,v_j) \not\in \mathcal{E(\tau)}, \tau = 0, \ldots, t-1&
    \end{aligned}
    \label{eq: degree}
\end{equation}
where $\lvert{\cdot}\rvert$ denotes the cardinality of a set and $\nu_i, \nu_j$ denote the velocities of the $i^\textrm{th}$ and $j^\textrm{th}$ vehicles, respectively.
\end{definition}

\noindent In this work, we use CMetric to compute the behavior profiles of traffic-agents, represented by a $n-$dimensional vector $\bm{\zeta},$ that is provided to the auction program.

\subsection{Sponsored search auctions (SSAs)}
\label{subsec: SSA_background}
Sponsored search auctions (SSAs) are a game-theoretic mechanism that are used extensively in internet search engines for the purpose of internet advertising~\cite{roughgarden2016twenty}. In an SSA, there are $K$ items to be allocated among $n$ agents. Each agent $a_i$ has a private valuation $\bm{v}_i$ and submits a bid $\bm{b}_i$ to receive at most one item of value $\bm{\alpha}_i$. A strategy is defined as an $n$ dimensional vector, $\bm{b} = (\bm{b}_i \cup \bm{b}_{-i})$, representing the bids made by every agent. $\bm{b}_{-i})$ denotes the bids made by all agents except $a_i$. Furthermore, let $\bm{b}_1 > \bm{b}_2 > \ldots > \bm{b}_K$ and $\bm{\alpha}_1 > \bm{\alpha}_2 > \ldots > \bm{\alpha}_K$. The allocation rule is that the agent with the $i^\textrm{th}$ highest bid is allocated the $i^\textrm{th}$ most valuable item, $\bm{\alpha}_i$. The utility $\bm{u}_i$~\cite{roughgarden2016twenty} incurred by $a_i$ is given as follows,

\begin{equation}
 \bm{u}_i (\bm{b}_i) =  \bm{v}_i \bm{\alpha}_i - \sum_{j=i}^k \bm{b}_{j+1} \left( \bm{\alpha}_j - \bm{\alpha}_{j+1} \right).
 \label{eq: utility_template}
\end{equation}

\noindent In the equation above, the quantity on the left represents the total utility for $a_i$ which is equal to value of the allocated goods $\bm{\alpha}_i$ minus a payment term. The first term on the right is the value of the item obtained by $a_i$. The second term on the right is the payment made by $a_i$ as a function of bids $\bm{b}_{j > i}$ and their allocated item values $\bm{\alpha}_j$. We refer the reader to Chapter $3$ in~\cite{roughgarden2016twenty} for a derivation and detailed analysis of Equation~\ref{eq: utility_template}.

In our approach, we re-cast Equation~\ref{eq: utility_template} through the lens of a human driver. More specifically, the term $\bm{v}_i \bm{\alpha}_i$ denotes the time reward gained by driver $a_i$ by moving on her turn. The payment term represents a notion of risk~\cite{wang2020game} associated with moving on that turn. It follows that an allocation of a conservative agent to a later turn (smaller $\bm{\alpha}$) also presents the lowest risk and vice-versa.

%% file: Sections/4-gameplan.tex
\section{Our Algorithm: \model}
\label{sec: gameplan}






Choosing an optimal ordering, in which agents navigate unsignaled and uncontrolled traffic scenarios, can be cast as an allocation problem where the goal is to allocate each agent, $a_i$, a position in the optimal turn-based ordering ($\bm{\sigma}_i$). Deciding such an allocation depends heavily on the incentives of the agents which, in the case of non-ideal agents, is a hard problem. Prior planning methods model non-ideal agents by estimating the objective functions of the agents from noisy data using statistical methods~\cite{gt1, schwarting2018planning, tian2020game} or by assuming a fixed behavior for surrounding agents (static or constant velocity)~\cite{tian2020game, li2020game}. These methods are not guaranteed to be optimal and result in collisions and deadlocks in unsignaled traffic scenarios, as shown in Table~\ref{tab: accuracy}.

Auction-based methods, on the other hand, model non-ideal agents in unsignaled traffic scenarios effectively albeit using a monetary-based bidding strategy that is not realizable in real-world scenarios~\cite{carlino2013auction,sayin2018information, lin2021pay, vasirani2012market}. Our formulation, \model, differs in this regard wherein we use a novel online driving behavior-based bidding strategy using the CMetric model~\cite{cmetric}. In the rest of this section, we present the main algorithm followed by an analysis of its optimality.

\subsection{Algorithm}
\label{subsec: using_auctions_for_multi}

Our goal is to solve Problem~\ref{problem: prob} and compute the optimal turn-based ordering $\bm{\sigma}_{\textsc{opt}} = \bm{\sigma}_1 \bm{\sigma}_2 \ldots \bm{\sigma}_n$, which shall determine the order in which agents will navigate unsignaled intersections, roundabouts, or merging. Our \model~algorithm proceeds in two stages: the behavior modeling phase and the planning phase.

During the behavior modeling phase, we use CMetric to compute the behavior profiles $\bm{\zeta}_i$ for every agent (active or non-active) using Equation~\ref{eq: cmetric} during an observation period of $5$ seconds. However, alternative behavior models such as SVO~\cite{schwarting2019social} may be used. This is followed by the planning phase which runs a sponsored search auction (SSA) scheme. In the auction scheme, each active agent has a private valuation $\bm{v}_i$. Each agent $a_i$ submits a bid $\bm{b}_i \in \mathbb{R}^{\geq 0}$ and obtains a time reward of $\frac{1}{t_i}$ for completing the navigation task in $t_i$ seconds, measured from the time the first agent begins to move. Note that moving earlier corresponds to a higher time reward. 

To summarize the algorithm, the agent with the highest bid \textit{i.e.} most aggressive behavior is allocated the highest priority and is allowed to navigate the scenario first, followed by the second-most aggressive, and so on. Therefore, 

\begin{equation}
    (\bm{\sigma}_\textsc{opt})_i = j^{\star},
    \label{eq: algorithm}
\end{equation} 

\noindent where $j^\star$ is the index of $\bm{\zeta}_{j^\star}$ in the sequence $\bm{\zeta}_1 > \bm{\zeta}_2 > \ldots > \bm{\zeta}_{j^\star} > \ldots > \bm{\zeta}_K$.


\subsection{Game-theoretic optimality and efficiency analysis}
\label{subsec: optimality}


In this section, we show that our approach is incentive compatible, welfare maximizing, and can be computed in polynomial time.

\subsubsection{Incentive compatibility}

The goal of any optimal auction should be such that that no agent is incentivised to ``cheat'' or, more simply, when the dominant strategy for each agent is to bid their true valuation $\bm{v}_i$. We define a dominant strategy as,

\begin{definition}
\textbf{Dominant Strategy:} Bidding $\bm{b}_i$ is a dominant strategy for $a_i$ if $\bm{u}_i(\bm{b}_i, \bm{b}_{-i}) > \bm{u}_i(\bm{\bar b}_i, \bm{b}_{-i})$ for all $ \bm{\bar b}_i \neq \bm{b}_i$.
\end{definition}

\noindent Ensuring fair allocations is crucial for auctions applied to traffic scenarios since unfair allocations could result in collisions and deadlocks. Incentivising traffic-agents to bid their true value as a dominant strategy is known as incentive compatibility~\cite{sayin2018information, rey2021online, roughgarden2016twenty, carlino2013auction} which is defined as follows,

\begin{definition}
\textbf{Incentive Compatibility:} An auction is said to be incentive compatible if for each agent, bidding $\bm{b}_i = \bm{v}_i$ is a dominant strategy.
\label{def: incentive_compatibility}
\end{definition}

\noindent In our formulation, we set the true valuation ($\bm{v}_i$) for a traffic-agent to be equal to its behavior profile $\bm{\zeta}_i$. We justify this decision in Section~\ref{subsec: alternate_strat}. Hence,

\[\bm{v}_i = \bm{\zeta}_i.\]

\noindent And so to show that our auction is incentive-compatible, we show the following,

\begin{theorem}
For each active agent $a_i \in \mathcal{A}$ at a traffic intersection, roundabout, or during merging, bidding $\bm{b}_i = \bm{\zeta}_i$ is the dominant strategy. 
\label{thm: incentive_compatibility}
\end{theorem}

\begin{proof}
Recall that the $k^\textrm{th}$ highest bidder ($k^\textrm{th}$ most aggressive agent) receives a time reward $\bm{\alpha}_k = \frac{1}{t_k}$. Then according to Equation~\ref{eq: utility_template}, the overall utility achieved by the $k^\textrm{th}$ most aggressive traffic-agent is,  

\begin{equation*}
  \bm{u}_k (\bm{b}_k) =  \bm{\zeta}_k \left( \frac{1}{t_k} \right) - \sum_{j=k}^K \bm{b}_{j+1} \left(\frac{1}{t_j} - \frac{1}{t_{j+1}} \right).
\end{equation*}

\noindent We sort the $K$ highest bids received in the following order: $\bm{b}_1 > \bm{b}_2 > \ldots > \bm{b}_K$. In order to show that $\bm{b}_k = \bm{\zeta}_k$ is the dominant strategy, it is sufficient to show that over-bidding ($\bm{\bar b}_k > \bm{\zeta}_k$) and under-bidding ($\bm{\bar b}_k < \bm{\zeta}_k$) both result in a lower utility than $\bm{u}_k$. We proceed by analyzing both cases.\\

\noindent \textbf{Case 1: Over-bidding ($\bm{\bar b}_k = \bm{b}_{k-1} > \bm{b}_k$):} In this case, the new utility for $a_k$ is $\bm{\bar u}_k (\bm{\bar b}_k)$ which is equal to,

\begin{equation}
     \bm{\zeta}_{k} \left( \frac{1}{t_{k-1}} \right) - \bm{b}_k \left( \frac{1}{t_{k-1}} - \frac{1}{t_{k}} \right ) - \sum_{j=k}^K \bm{b}_{j+1} \left(\frac{1}{t_j} - \frac{1}{t_{j+1}} \right).
     \label{eq: utility_overbidding}
\end{equation}

\noindent From Equation~\ref{eq: utility_template} and Equation~\ref{eq: utility_overbidding}, the net \ul{increase} in utility is,

\begin{equation}
    \bm{\bar u}_k (\bm{\bar b}_k) - \bm{u}_k (\bm{b}_k) =  \left(\bm{\zeta}_{k} - \bm{b}_k \right) \left( \frac{1}{t_{k-1}} - \frac{1}{t_{k}}\right).
    \label{eq: increase_in_utility}
\end{equation}

\noindent Therefore, bidding $\bm{\bar b}_k > \bm{\zeta}_k \implies \bm{\bar u}_k (\bm{\bar b}_k) - \bm{u}_k (\bm{b}_k) <0$ since $t_{k-1} < t_k$. In other words, overbidding yields negative utility for agent $a_k$. \\

\noindent \textbf{Case 2: Under-bidding ($\bm{\bar b}_k = \bm{b}_{k+1} < \bm{b}_i$):} The new utility in this case is given by,

\begin{equation}
    \bm{\bar u}_k (\bm{\bar \zeta}_k) =  \bm{\zeta}_{k} \left( \frac{1}{t_{k+1}} \right) - \sum_{j=k+1}^K \bm{b}_{j+1} \left(\frac{1}{t_j} - \frac{1}{t_{j+1}} \right)
    \label{eq: utility_underbidding}
\end{equation}

\noindent From Equation~\ref{eq: utility_template} and Equation~\ref{eq: utility_underbidding}, the net \ul{decrease} in utility is,

\begin{equation}
\bm{u}_k (\bm{b}_k) - \bm{\bar u}_k (\bm{\bar b}_k) =   \left(\bm{\zeta}_{k} + \bm{b}_{k+1}\right )  \left( \frac{1}{t_{k}} - \frac{1}{t_{k+1}}\right ).
\label{eq: decrease_in_utility}
\end{equation}

\noindent Note that Equation~\ref{eq: decrease_in_utility} is always positive since $\bm{\zeta}, \bm{b}_{k+1} > 0$ and $t_k < t_{k+1}$. This implies that under-bidding always results in a decrease in utility as well.\\
\end{proof}

\subsubsection{Welfare maximization}

The next desired property in an optimal auction is welfare maximization~\cite{roughgarden2016twenty, sayin2018information} which maximizes the total utility earned by every active agent. 



\begin{theorem}
\textbf{Welfare maximization:} Social welfare of an auction is defined as $\sum_i \bm{v}_i\bm{\alpha}_i$. Welfare maximization involves finding the strategy $\bm{b}$ that maximizes $\sum_i \bm{v}_i\bm{\alpha}_i$. For each active agent $a_i \in \mathcal{A}$, bidding $\bm{b}_i = \bm{\zeta}_i$ maximizes social welfare. 
\label{thm: Welfare_Maximizing}
\end{theorem}


\begin{proof}
Our proof is based on induction. We begin with the base case with the most aggressive agent (highest bidder). Recall that after sorting, we have agents in decreasing order of aggressiveness \textit{i.e.} $\bm{\zeta}_1 > \bm{\zeta}_2>\ldots> \bm{\zeta}_n$ and $\frac{1}{t_1} > \frac{1}{t_2}>\ldots> \frac{1}{t_k}$. Therefore, we have that $\frac{\bm{\zeta}_1}{t_1}$ is maximum. Next, consider the hypothesis that the sum $\sum_{j=1}^k\frac{\bm{\zeta}_j}{t_j}$ is maximum up to the $k^\textrm{th}$ highest bidder. Then the inductive step is to prove that $\sum_{j=1}^{k+1} \frac{\bm{\zeta}_j}{t_j}$ is maximum. Observe that,

\[\sum_{j=1}^{k+1} \left( \frac{\bm{\zeta}_j}{t_j}\right) = \sum_{j=1}^k \left(\frac{\bm{\zeta}_j}{t_j}\right) + \frac{\bm{\zeta}_{k+1}}{t_{k+1}}\]

\noindent Note that the first term on the RHS is maximum from hypothesis. Then, 

\begin{equation}
    \begin{split}
        \bm{\zeta}_{k+1} > \bm{\zeta}_{k+2} > \ & \bm{\zeta}_{k+3}>\ldots> \bm{\zeta}_n \\
        &\textrm{and}\\
        \frac{1}{t_{k+1}} > \frac{1}{t_{k+2}}> \ & \frac{1}{t_{k+3}}>\ldots> \frac{1}{t_K}\\
    \end{split}
\end{equation} 

\noindent implies that $\frac{\bm{\zeta}_{k+1}}{t_{k+1}}$ is maximum.\\
\end{proof}

\subsubsection{Polynomial time computation}

Finally, in terms of planning and auction design~\cite{roughgarden2016twenty}, it is important to show that the underlying auction is computationally efficient and can handle a large number of agents. We show that our approach runs in polynomial time via the following theorem,

\begin{theorem}
\textbf{Polynomial Runtime:} \model~runs in polynomial time.
\label{thm: Polynomial_Runtime}
\end{theorem}
\begin{proof}
The main computation in our algorithm is dominated by sorting the agent's CMetric values; it is known that sorting algorithms run in polynomial time~\cite{clrs}.
\end{proof}

\subsection{Using $\sigma_{\textsc{opt}}$ for collision prevention and deadlock resolution}

We identify a deadlock as a situation when two or more active traffic-agents remain stationary for an extended period of time due to the uncertainty in the actions of other active traffic-agents. Deadlocks may arise in traffic scenarios consisting of multiple conservative and/or aggressive agents and are resolved when one of the agents opts to move based on some heuristic. Via $\bm{\sigma}_{\textsc{opt}}$, agents automatically know when each agent is supposed to move thereby eliminating any confusion or uncertainty in the actions of other agents.

\model~can also prevent collisions in a similar manner. The number of collisions, or the likelihood thereof, increases when two or more aggressive agents decide to \texttt{MOVE} first, simultaneously, despite the uncertainty in the actions of the other agents. $\bm{\sigma}_{\textsc{opt}}$ can break ties between multiple aggressive drivers since $\bm{\zeta}_i \neq \bm{\zeta}_j$ for $i,j \in [1,n]$. Using the turn-based ordering determined by $\bm{\sigma}_{\textsc{opt}}$, less aggressive agents can let more aggressive agents pass first (as shown in Figure~\ref{fig: cover}).

%% file: Sections/Theory.tex
\section{Analysis of different bidding strategies}
\label{sec: np-complete}
\input{Tables/comp_auctions}
\input{{Tables/accuracy}}

At this point, there are two concerns that may come to the mind of the reader:

\begin{enumerate}
    \item Justification of using driver behavior as the bidding strategy as opposed to using monetary budgets, arrival times, or any random driver-specific property, \textit{e.g.} height.
    
    \item Since the optimal solution in \model~is to reward more aggressive drivers, conservative agents may be motivated to act aggressively.
\end{enumerate}

\noindent In the following, we formally address both issues.



\subsection{Justifying behavior-based bidding for human drivers with social preferences}
\label{subsec: alternate_strat}

The key insight to achieving an optimal auction is that agents' bidding strategy ($\bm{b}$) should be correlated with their social preferences or private valuations ($\bm{v}$). That is,

\begin{equation}
\begin{split}
    \bm{b}_1 > \bm{b}_2& > \ldots > \bm{b}_K\\
    &\iff\\
    \bm{v}_1 > \bm{v}_2& > \ldots > \bm{v}_K\\
\end{split}
\end{equation}

\noindent \hllb{where $\iff$ denotes equivalency} and with optimality occurring when $\bm{b}_i = \bm{v}_i = \bm{\zeta}_i$. In the case of behavior-based bidding, we explicitly enforce the above constraints via \model. Economic auctions, random bidding, and FIFO-based strategies, on the other hand, disregard the aspect of human social preference ($\bm{v}$). A simple analysis reveals that ignoring private valuations can result in auctions where agents are encouraged to falsify their bids.

Consider any economic auction~\cite{carlino2013auction}. Let the cash bids made by agents be $\bm{b}_1 > \bm{b}_2 > \ldots > \bm{b}_K$ with associated time rewards, $\bm{\alpha}_1 > \bm{\alpha}_2 > \ldots > \bm{\alpha}_K$. Now assume a single agent $a_k$ is an impatient driver with a budget $B_k$ and bid $\bm{b}_k \leq B_k$ ($B_k < B_m$ where $m$ indexes those agents whose bid $\bm{b}_m > \bm{b}_k$). Since economic auctions do not take social preferences or private valuations into account, it can be shown that $\bm{v}_k > \bm{b}_k$ for the impatient agent $a_k$. To see how, recall that $a_k$'s private valuation $\bm{v}_k$ is much higher than the agents before them. That is, $\exists m<k$ such that $\bm{v}_k > \bm{v}_m$. Without loss of generality, assume $m=k-1$. Then we have that $\bm{v}_1 > \ldots>\bm{v}_k> \bm{v}_{k-1}>\ldots> \bm{v}_K$. Since $\bm{v}_k>\bm{v}_{k-1}$ but $\bm{b}_{k-1}>\bm{b}_k$ and $\bm{v}_{k-1} = \bm{b}_{k-1}$, $\bm{v}_k > \bm{b}_{k}$ and $a_k$ can decide to increase her bid to $\bm{\bar b}_k = \bm{b}_{k-1} + \epsilon >\bm{b}_{k-1}$. The updated utility for $a_k$ is,

\begin{equation*}
    \bm{\bar u}_k (\bm{\bar b}_k) = \bm{v}_k \left( \frac{1}{t_{k-1}} \right ) - \sum_{j=k-1}^K \bm{b}_{j+1} \left(\frac{1}{t_j} - \frac{1}{t_{j+1}} \right).
\end{equation*}

\noindent The difference from original utility is,

\begin{equation}
    \bm{d} = \left( \bm{v}_k - \bm{b}_k \right)  \left( \frac{1}{t_{k-1}} - \frac{1}{t_{k}} \right ).
    \label{eq: difference}
\end{equation}

\noindent Since $\bm{v}_k > \bm{b}_k$ and $t_{k-1} < t_{k}$, the net change is positive. Therefore, agents gain utility by over-bidding. This implies that economic and FIFO-based auctions are not optimal. We present empirical evidence in Section~\ref{subsec: comparison_with_prior_methods}.

\subsection{Behavior-based bidding does not encourage conservative agents to act aggressively}
\label{subsec: behavior_and_conservative}
At first glance, it may seem that a similar argument can be made against our approach wherein it is biased towards aggressive drivers by rewarding them with higher priority, thereby encouraging conservative drivers to be more aggressive (over-bidding). However, this is not the case since in the case of behavior-based bidding, acting aggressively and therefore, over-bidding $\bm{b}_k > \bm{v}_k \implies \bm{d} < 0$ by Equation~\ref{eq: difference}. In other words, conservative drivers lose utility by acting aggressively. \hllg{This is because in the event of an overbid, the gain in time reward, $\bm{\zeta}_{k}\left( \frac{1}{t_{k-1}} - \frac{1}{t_{k}}\right)$ is \textit{less} than the risk penalty, $\bm{b}_{k}\left( \frac{1}{t_{k-1}} - \frac{1}{t_{k}}\right)$. Special cases include when two or more agents decide to over bid at the same time. The driver behavior modeling algorithm, CMetric, automatically takes care of these issues as it outputs a real number ($\mathbb{R}$). In the unlikely event that the CMetric output for multiple vehicles be identical, we choose to break ties randomly.}

%% file: Tables/comp_auctions.tex
\begin{table*}[t]
{\fontsize{9pt}{9pt}\selectfont
\caption{\textbf{Using different bidding strategies:} We compare the rate of collisions by replacing the CMetric bidding strategy with economic auctions, first-in-first-out (FIFO), and random bidding in a range of simulation settings that include changing the traffic density, number of aggressive drivers, likelihood of a driver being aggressive, and speed limit.}
\label{tab: auctions}
\centering
\fontsize{8}{10}\selectfont{\resizebox{\textwidth}{!}{
\begin{tabular}{@{}ccccc|cc|cccc|cc@{}}
  \toprule
  \multirow{2}{*}{Strategy} &
  \multicolumn{4}{c}{\% Aggressive}&
  \multicolumn{2}{c}{\# Agg. Agents}& 
  \multicolumn{4}{c}{\# Agents \fontsize{7}{8}\selectfont{(w. 1 agg. agent)}} & 
  \multicolumn{2}{c}{Max. Speed} \\
  \cmidrule{2-5} 
  \cmidrule{6-7} 
  \cmidrule{8-11}
  \cmidrule{12-13}
   & $20\%$ & $25\%$ & $33\%$ & $50\%$ &  $1$ & $2$  &  $4$ & $6$ & $8$ & $10$ & $45$mph & $60$mph \\
\midrule

  Economic & 
    $2.48$ & $2.99$ & $6.55$ &
    $8.29$& $2.81$ & $6.18$ & 
    $3.07$ & $3.27$ & $3.58$ & $3.62$ &$2.99$& $3.07$ \\ 
  FIFO & 
    $7.18$ & $8.69$ & $10.20$ & 
    $15.87$& $7.65$ & $15.06$ & 
    $7.59$ & $8.36$ & $8.65$ & $8.91$ & $7.72$ & $7.47$ \\ 
  Random  & 
    $14.79$ & $18.67$ & $25.24$ &
    $37.77$ & $15.01$ & $28.91$ & 
    $14.82$ & $16.69$ & $17.59$  & $17.90$& $15.01$&$15.01$\\ 
  \textbf{\model}  & 
     $\textbf{0.33}$ & $\textbf{0.37}$ & $\textbf{0.63}$ & 
    $\textbf{0.79}$& $\textbf{0.29}$ & $\textbf{0.66}$ & 
    $\textbf{0.12}$ & $\textbf{0.24}$ & $\textbf{0.37}$  & $\textbf{0.64}$& $\textbf{0.44}$&$\textbf{0.53}$\\ 

  \bottomrule
\end{tabular}}
}
}
  \vspace{-5pt}
\end{table*}

%% file: Tables/accuracy.tex
\begin{table}[t]
\caption{\textbf{Comparing with other approaches:} We compare the percentage \hllb{(\%)} of collisions (C), deadlocks (D), and overall success rate (S) with state-of-the-art planning methods designed for uncontrolled and unsignaled intersections. The \xmark~symbol indicates that results for that setting were not provided by the method.
}
\centering
\resizebox{\columnwidth}{!}{
\begin{tabular}{rccccccccc}
  \toprule[1.5pt]
 \multirow{2}{*}{Methods}& \multicolumn{3}{c}{$2$ Vehicles} &
  \multicolumn{3}{c}{$3$ Vehicles}&  
  \multicolumn{3}{c}{$4$ Vehicles}\\
   & {C } & {D } & S  &  {C } & {D } & S  &  {C } & {D } & S \\
\midrule

  Capasso et al.~\cite{capasso2021end} & 
    \xmark & \xmark & \xmark & 
    $02$ & \xmark & \xmark & 
    $0$ & \xmark & \xmark  \\ 
  Liu et al.~\cite{liu2020decision} & 
    \xmark & \xmark & \xmark & 
    \xmark & \xmark & \xmark & 
    $6.2$ & \xmark & $93.7$  \\ 
  Isele et al.~\cite{isele2018navigating} & 
    \xmark & \xmark & \xmark & 
    \xmark & \xmark & \xmark & 
    $7.1$ & $0.2$ & $92.8$  \\ 
Kai et al.~\cite{Kai2020AMR} & 
    $0$ & $0$ & $100$ & 
    \xmark & \xmark & \xmark & 
    $3.7$ & $0 $ & $96.3$  \\ 
    
  Li et al.~\cite{li2020game} & 
    $0$ & $0$ & $100$ & 
    \xmark & \xmark & \xmark & 
    $2.7$ & $0.4$ & $96.9$  \\ 
  Roh et al.~\cite{roh2020multimodal} &
    $2$ & \xmark & $98$& $2$ & \xmark & $98$ & $30$ & \xmark & $70$   \\ 
  Tian et al.~\cite{tian2020game} &
    $10$ & $5$ & $85$ & 
    $10$ & $5$ & $85$ & 
    $10$ & $10$ & $80$  \\ 
  \model  & 
    $\textbf{0}$ & $\textbf{0}$ & $\textbf{100}$ & 
    $\textbf{0}$ & $\textbf{0}$ & $\textbf{100}$ & 
    $\textbf{0.1}$ & $\textbf{0.1}$ & $\textbf{99.8}$  \\ 
  \midrule
  
  \multirow{2}{*}{Methods}&
  \multicolumn{3}{c}{$6$ Vehicles}& 
  \multicolumn{3}{c}{$8$ Vehicles}&  
  \multicolumn{3}{c}{$10$ Vehicles}\\
   & {C } & {D } & S  &  {C } & {D } & S  &  {C } & {D } & S \\
  \midrule
    Capasso et al.~\cite{capasso2021end} & 
    \xmark & \xmark & \xmark &
    $03$ & \xmark & \xmark &
    \xmark & \xmark & \xmark  \\
  Li et al.~\cite{li2020game} &
    $10$ & $10$ & $90$ &
    $20$ & $20$ & $80$ &
    $20$ & $20$ & $80$  \\ 
  Tian et al.~\cite{tian2020game} &
    $10$ & $10$ & $80$ &
    $10$ & $10$ & $80$ &
    $20$ & $20$ & $60$  \\ 
    
    \midrule
  \model  &
    $\textbf{0.3}$ & $\textbf{0.1}$ & $\textbf{99.6}$ &
    $\textbf{0.4}$ & $\textbf{0.1}$ & $\textbf{99.5}$ &
    $\textbf{0.6}$ & $\textbf{0.2}$ & $\textbf{99.2}$  \\

  \bottomrule[1.5pt]
\end{tabular}
}
\label{tab: accuracy}
  \vspace{-10pt}
 
\end{table}

%% file: Sections/6-experiments.tex
\section{Experiments and Results}
\label{sec: experiments_and_results}

We compare our approach with state-of-the-art planning methods designed for unsignaled traffic scenarios. We also demonstrate \model~in two real-world merging scenarios involving human drivers. Additional results \hllb{from} roundabouts and merging can be found in the supplementary report~\hllb{\mbox{\cite{Supp_Report}}}.

\subsection{Experiment setup}
\label{subsec: experiment_setup}


We apply \model~in three scenarios-- intersection, merging, and roundabout. We use the OpenAI traffic simulator~\cite{leurent2019social} for computing the behavior profiles using CMetric during an observation period of $5$ seconds with a traffic density of $20$ vehicles in a $4800 m^2$ area. 
We average the number of collisions and deadlocks over the total number of simulations.


    


We compare \model~with a wide range of techniques developed for navigating unsignaled intersections, roundabouts, and merging scenarios which were discussed in Section~\ref{sec: related_work} and Table~\ref{tab: bidding_rules}. The methods based on deep learning, game theory, and DRL do not have publicly available open-sourced code; we compared our approach by matching their configuration using the exact same motion models and dynamic parameters (using the IDM model with identical velocity and acceleration constraints, equal number of lanes, similar traffic density etc.). We implemented the methods based on FIFO, economic auctions, and random bidding within our own simulation framework; these methods were easily reproducible due to their similarity to our approach.

\input{img/Real_World}

\subsection{Comparison with prior methods}
\label{subsec: comparison_with_prior_methods}


\subsubsection{Auctions} In Table~\ref{tab: auctions}, we compare our method in terms of collision rate with three auction formats using different bidding strategies--monetary or economic bidding, time-based bidding (FIFO), and random bidding. Our scenario consists of a $4$-way multi-lane intersection consisting of $4-10$ vehicles. Using the OpenAI simulator~\cite{leurent2019social}, we program agents to assume a particular behavior (conservative versus aggressive). We test in a wide range of configurations including the likelihood of a driver being aggressive, number of aggressive drivers (one versus two), total traffic density, and maximum speed (MS). The central argument we make in Section~\ref{subsec: alternate_strat} is that economic and FIFO-based auctions as well as random bidding strategies result in collisions when taking human preference or valuation ($\bm{v}$) into account. In Table~\ref{tab: auctions}, we show that our behavior-based auction significantly outperforms prior auctions when considering human preference.

We infer some additional results from Table~\ref{tab: auctions}. The ``\% Aggressive'' column corresponds to the likelihood of an instantiated driver being aggressive. For example, $25\%$ aggressiveness implies that $1$ in $4$ agents created by the simulator will be programmed as an aggressive agent. It is not surprising to observe that the rate of collision increases with the likelihood of aggressiveness. A similar reasoning can be made for the number of aggressive agents simulated, number of total agents, and the maximum speed allowed. Regarding the last factor, we program an agent to be aggressive if their speed exceeds the maximum speed at the moment it reaches an intersection. However, variation in the speed limits do not impact the collision-rate. Finally, we observe that economic auctions outperform FIFO and random bidding strategies.

\subsubsection{DRL, game theory, RNNs} To thoroughly evaluate our approach with respect to the state-of-the-art, we also compare \model~with alternatives to auctions namely approaches based on deep learning, DRL, and game theory in terms of collision rate (C), deadlock rate (D), and success rate (S) in Table~\ref{tab: accuracy}. Again, the scenario used for evaluation is a $4$-way multi-lane intersection because all the baselines use $4$-way intersections in their experiments as well. In order to convey a fair and uniform comparison, we identify a common configuration across all baselines. This configuration includes static environment features such as traffic density, speed limits, and number of lanes along with dynamic parameters such as velocity and acceleration constraints, lane-changing model, and motion and braking models. A few consequences of maintaining such a uniform configuration is that results for many settings are unavailable (marked with \xmark).

One caveat is that since these baselines do not guarantee collision-free navigation, the results in Table~\ref{tab: accuracy} are not surprising. The key benefit of our approach over these methods is that our results are game-theoretically optimal for $n$ agents in a $4-$way traffic intersections. That is, as $n$ increases, our approach can scale accordingly and we can still guarantee $0$ collisions and deadlocks from a planning standpoint. We observe that the number of collisions increases for the approaches of Li et al.~\cite{li2020game} and Roh et al.~\cite{roh2020multimodal} as the number of agents and service lanes in the intersections increase. 

\subsubsection{Computational time comparison} In terms of computational time~\mbox{\cite{vasirani2012market, censi2019today, lin2021pay}} are not game-theoretically optimal and therefore do not operate in polynomial time (Theorem~\mbox{\ref{thm: Polynomial_Runtime}}). However, auctions such as~\mbox{\cite{carlino2013auction, rey2021online, sayin2018information}} are both incentive compatible and run in polynomial time. For these last three works, despite having a similar computational time cost as our approach, our approach results in fewer collisions and deadlocks with aggressive and impatient drivers. Furthermore, we also compare our method with Kai et al.~\mbox{\cite{Kai2020AMR}} in terms of computational time in real world application. Kai et al.~\mbox{\cite{Kai2020AMR}} use multi-task deep reinforcement learning which takes on average $\bm{12}$ seconds to turn at an intersection. Our method, in comparison, takes $\bm{7}$ seconds to turn at the intersection.

\subsection{Demonstrating \model~with human drivers}
\label{subsec: real_world_exps}

A major distinction of our work from prior methods is the ease with which it can be applied in real-world scenarios. We demonstrate \model~in a merging scenario where one vehicle merges onto the road from a driveway (Figure~\ref{fig: real_world}). In this experiment, we recruit three participants driving a white, silver, and black vehicle, respectively. Each participant was given a goal and a set of actions. For example, the white vehicle is asked to drive straight and reach a final destination. The black and silver vehicles are asked to merge onto the road and turn right. Recall that each participant is uninformed about the actions of other agents. In fact, we do not tell participants that other participants are involved in the scenario. In other words, each participant is made aware of the other only after the interaction has started.

We assume that the behavior profiles of all participants are given or we can compute them using their trajectories. Each participant is asked to drive a fixed path with other traffic agents. We use the CMetric algorithm to compute the $\bm{\zeta}$ values for each participant and ask them to privately verify the value obtained by the CMetric algorithm. Post verification, we recorded the drivers of the white, black, and silver vehicles to be aggressive, conservative, and neutral, respectively. The scenario consists of the white vehicle approaching from the left towards the end of the driveway while the black and silver vehicles exit the driveway and merge onto the road. We record the interactions and time-to-goal (TTG) of the three agents with and without using \model. The white and black agents participate in the former, while the white and the silver agent participate in the latter.

When agents using \model~to plan their strategies, the white agent is aggressive and knows that the black agent is conservative. Therefore, the white agent is incentivised to continue moving forward and the black agent yields as expected. Next, we compare the interaction between the white and the silver agents without \model. Recall that the white agent does not know the strategy of the silver agent and in this case, is also unaware of the behavior profile of the silver agent, and vice-versa. Therefore, both agents stop at the intersection, entering into a deadlock. This is sub-optimal as it leads to a TTG delay of $8$ seconds. Note that since the white agent is inherently a more aggressive driver, they are first to resume moving after the delay.



%% file: img/Real_World.tex
\begin{figure*}[t]
\centering
    \begin{subfigure}[h]{\linewidth}
    \includegraphics[width=\textwidth]{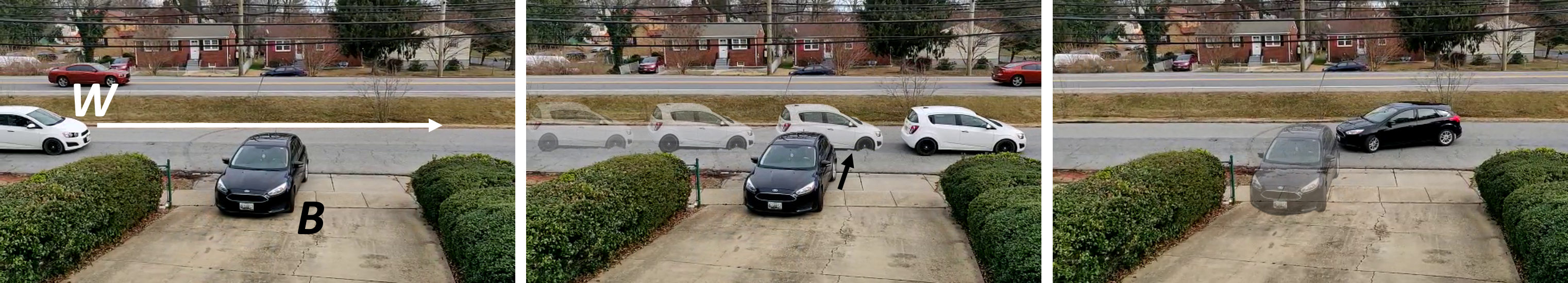}
    \caption{$\sigma_W = 1, \sigma_B = 2$, TTG $= 7$ seconds.}
    \label{fig: white_black}
    \vspace{2pt}
  \end{subfigure}
  \begin{subfigure}[h]{\linewidth}
    \includegraphics[width=\textwidth]{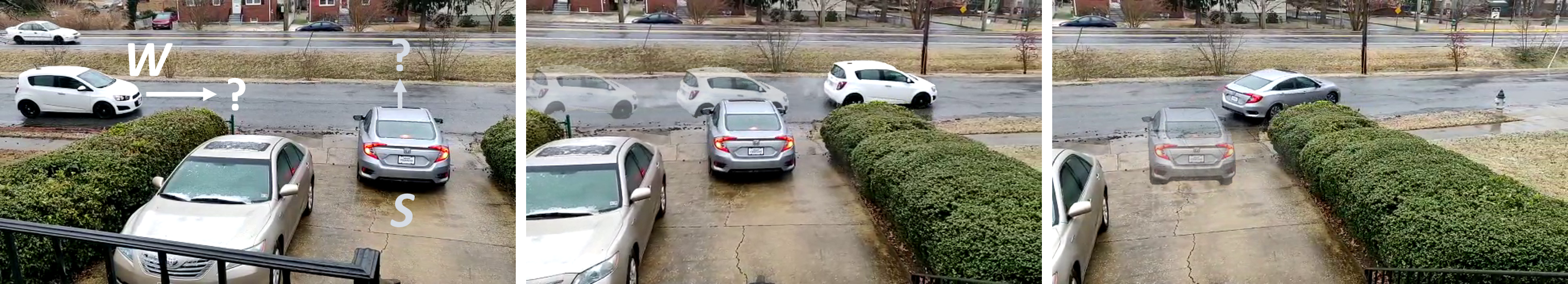}
    \caption{$\sigma_W = \sigma_S = \textrm{\xmark}$, TTG $= 15$ seconds.}
    \label{fig: white_silver}
  \end{subfigure}
  %
  %
%
  %
\caption{\textbf{Demonstrating \model~with Human Drivers:} In Figure~\ref{fig: white_black}, the driver of the white (W) vehicle is aggressive while the driver of the black (B) vehicle is conservative. \model~computes a turn-based ordering that determines the black agent should yield to the white vehicle, which would pass first. In Figure~\ref{fig: white_silver}, we do not perform game-theoretic planning and consequently, both drivers are unsure of who should move first. In this case, both the white (W) and silver (S) agents stop and after a delay of $8$ seconds, the white vehicle moves first (due to the inherently aggressive nature of the driver) followed by the silver vehicle.}
  \label{fig: real_world}
  \vspace{-5pt}
\end{figure*}

%% file: Sections/7-conclusion.tex
\section{Conclusion, Limitations, and Future Work}
\label{sec: conclusion}

We present a novel multi-agent game-theoretic planning algorithm called \model~in intersections, roundabouts, and during merging with human drivers and autonomous vehicles. \model~uses the behavior profiles of all traffic-agents, combines with sponsored search auctions, and produces an optimal turn-based ordering. We show that \model~is incentive compatible, welfare maximizing, and operates in polynomial time. We reduce the number of collisions and deadlocks by at least $10-20\%$ on average over prior methods. Moreover, we demonstrate \model~in two merging scenarios involving real human drivers and show that our game-theoretic model is applicable in real-world scenarios.

There are a few limitations of our work. Our approach is primarily designed for moderately to highly dense traffic as CMetric~\cite{cmetric} may not work as well in sparse traffic conditions. In such cases, data-driven behavior models such as SVO~\cite{schwarting2019social} may be used. There are many interesting directions for future work. For example, our method currently does not plan beyond computing turn-based orderings, i.e. local navigation. Next steps may include integrating \model~with global motion planning methods to achieve an end-to-end navigation approach for non-communicating multi-agent traffic scenarios. \hllb{In addition, we have currently demonstrated real world application with $2-3$ vehicles. In the future, we plan to conduct further evaluation in denser and more comprehensive real world settings with more vehicles.}

%% file: main.bbl
\begin{thebibliography}{10}

\bibitem{gt6}
Noam Buckman, Alyssa Pierson, Wilko Schwarting, Sertac Karaman, and Daniela~L
  Rus.
\newblock Sharing is caring: Socially-compliant autonomous intersection
  negotiation.
\newblock 2020.

\bibitem{capasso2021end}
Alessandro~Paolo Capasso, Paolo Maramotti, Anthony Dell'Eva, and Alberto
  Broggi.
\newblock End-to-end intersection handling using multi-agent deep reinforcement
  learning.
\newblock {\em arXiv preprint arXiv:2104.13617}, 2021.

\bibitem{carlino2013auction}
Dustin Carlino, Stephen~D Boyles, and Peter Stone.
\newblock Auction-based autonomous intersection management.
\newblock In {\em 16th International IEEE Conference on Intelligent
  Transportation Systems (ITSC 2013)}, pages 529--534. IEEE, 2013.

\bibitem{censi2019today}
Andrea Censi, Saverio Bolognani, Julian~G Zilly, Shima~Sadat Mousavi, and
  Emilio Frazzoli.
\newblock Today me, tomorrow thee: Efficient resource allocation in competitive
  settings using karma games.
\newblock In {\em 2019 IEEE Intelligent Transportation Systems Conference
  (ITSC)}, pages 686--693. IEEE, 2019.

\bibitem{cmetric}
Rohan Chandra, Uttaran Bhattacharya, Trisha Mittal, Aniket Bera, and Dinesh
  Manocha.
\newblock Cmetric: A driving behavior measure using centrality functions.
\newblock {\em arXiv preprint arXiv:2003.04424}, 2020.

\bibitem{cleac2019algames}
Simon~Le Cleac'h, Mac Schwager, and Zachary Manchester.
\newblock Algames: A fast solver for constrained dynamic games.
\newblock {\em arXiv preprint arXiv:1910.09713}, 2019.

\bibitem{cleac2020lucidgames}
Simon~Le Cleac'h, Mac Schwager, and Zachary Manchester.
\newblock Lucidgames: Online unscented inverse dynamic games for adaptive
  trajectory prediction and planning.
\newblock {\em arXiv preprint arXiv:2011.08152}, 2020.

\bibitem{clrs}
Thomas~H Cormen, Charles~E Leiserson, Ronald~L Rivest, and Clifford Stein.
\newblock {\em Introduction to algorithms}.
\newblock 2009.

\bibitem{dresner2008multiagent}
Kurt Dresner and Peter Stone.
\newblock A multiagent approach to autonomous intersection management.
\newblock {\em Journal of artificial intelligence research}, 31:591--656, 2008.

\bibitem{gt1}
Jaime~F Fisac, Eli Bronstein, Elis Stefansson, Dorsa Sadigh, S~Shankar Sastry,
  and Anca~D Dragan.
\newblock Hierarchical game-theoretic planning for autonomous vehicles.
\newblock In {\em 2019 International Conference on Robotics and Automation
  (ICRA)}, pages 9590--9596. IEEE, 2019.

\bibitem{grembek2018introducing}
Offer Grembek, Alex~A Kurzhanskiy, Aditya Medury, Pravin Varaiya, and Mengqiao
  Yu.
\newblock Introducing an intelligent intersection.
\newblock {\em ITS Reports}, 2018(13), 2018.

\bibitem{isele2018navigating}
David Isele, Reza Rahimi, Akansel Cosgun, Kaushik Subramanian, and Kikuo
  Fujimura.
\newblock Navigating occluded intersections with autonomous vehicles using deep
  reinforcement learning.
\newblock In {\em 2018 IEEE International Conference on Robotics and Automation
  (ICRA)}, pages 2034--2039. IEEE, 2018.

\bibitem{Kai2020AMR}
Shixiong Kai, Bin Wang, D.~Chen, Jianye Hao, Hongbo Zhang, and Wulong Liu.
\newblock A multi-task reinforcement learning approach for navigating
  unsignalized intersections.
\newblock {\em 2020 IEEE Intelligent Vehicles Symposium (IV)}, pages
  1583--1588, 2020.

\bibitem{leurent2019social}
Edouard Leurent and Jean Mercat.
\newblock Social attention for autonomous decision-making in dense traffic.
\newblock {\em arXiv preprint arXiv:1911.12250}, 2019.

\bibitem{li2020game}
Nan Li, Yu~Yao, Ilya Kolmanovsky, Ella Atkins, and Anouck~R Girard.
\newblock Game-theoretic modeling of multi-vehicle interactions at uncontrolled
  intersections.
\newblock {\em IEEE Transactions on Intelligent Transportation Systems}, 2020.

\bibitem{lin2021pay}
DianChao Lin and Saif~Eddin Jabari.
\newblock Pay for intersection priority: A free market mechanism for connected
  vehicles.
\newblock {\em IEEE Transactions on Intelligent Transportation Systems}, 2021.

\bibitem{liu2020decision}
Teng Liu, Xingyu Mu, Bing Huang, Xiaolin Tang, Fuqing Zhao, Xiao Wang, and
  Dongpu Cao.
\newblock Decision-making at unsignalized intersection for autonomous vehicles:
  Left-turn maneuver with deep reinforcement learning.
\newblock {\em arXiv preprint arXiv:2008.06595}, 2020.

\bibitem{bgap}
Angelos Mavrogiannis, Rohan Chandra, and Dinesh Manocha.
\newblock B-gap: Behavior-guided action prediction for autonomous navigation.
\newblock {\em arXiv preprint arXiv:2011.03748}, 2020.

\bibitem{dql}
Volodymyr Mnih, Koray Kavukcuoglu, David Silver, Andrei~A Rusu, Joel Veness,
  Marc~G Bellemare, Alex Graves, Martin Riedmiller, Andreas~K Fidjeland, Georg
  Ostrovski, et~al.
\newblock Human-level control through deep reinforcement learning.
\newblock {\em nature}, 518(7540):529--533, 2015.

\bibitem{rey2021online}
David Rey, Michael~W Levin, and Vinayak~V Dixit.
\newblock Online incentive-compatible mechanisms for traffic intersection
  auctions.
\newblock {\em European Journal of Operational Research}, 2021.

\bibitem{rodrigues2019network}
Francisco~Aparecido Rodrigues.
\newblock Network centrality: An introduction.
\newblock {\em A Mathematical Modeling Approach from Nonlinear Dynamics to
  Complex Systems}, page 177, 2019.

\bibitem{roh2020multimodal}
Junha Roh, Christoforos Mavrogiannis, Rishabh Madan, Dieter Fox, and
  Siddhartha~S Srinivasa.
\newblock Multimodal trajectory prediction via topological invariance for
  navigation at uncontrolled intersections.
\newblock {\em arXiv preprint arXiv:2011.03894}, 2020.

\bibitem{Supp_Report}
Dinesh~Manocha Rohan~Chandra.
\newblock Supplementary report for gameplan: Game-theoretic multi-agent
  planning with human drivers at intersections, roundabouts, and merging.
\newblock \url{https://arxiv.org/pdf/2109.01896.pdf}, 2021.

\bibitem{roughgarden2016twenty}
Tim Roughgarden.
\newblock {\em Twenty lectures on algorithmic game theory}.
\newblock Cambridge University Press, 2016.

\bibitem{sayin2018information}
Muhammed~O Sayin, Chung-Wei Lin, Shinichi Shiraishi, Jiajun Shen, and Tamer
  Ba{\c{s}}ar.
\newblock Information-driven autonomous intersection control via incentive
  compatible mechanisms.
\newblock {\em IEEE Transactions on Intelligent Transportation Systems},
  20(3):912--924, 2018.

\bibitem{schepperle2007agent}
Heiko Schepperle and Klemens B{\"o}hm.
\newblock Agent-based traffic control using auctions.
\newblock In {\em International Workshop on Cooperative Information Agents},
  pages 119--133. Springer, 2007.

\bibitem{schwarting2018planning}
Wilko Schwarting, Javier Alonso-Mora, and Daniela Rus.
\newblock Planning and decision-making for autonomous vehicles.
\newblock {\em Annual Review of Control, Robotics, and Autonomous Systems},
  2018.

\bibitem{schwarting2019social}
Wilko Schwarting, Alyssa Pierson, Javier Alonso-Mora, Sertac Karaman, and
  Daniela Rus.
\newblock Social behavior for autonomous vehicles.
\newblock {\em Proceedings of the National Academy of Sciences},
  116(50):24972--24978, 2019.

\bibitem{tian2020game}
Ran Tian, Nan Li, Ilya Kolmanovsky, Yildiray Yildiz, and Anouck~R Girard.
\newblock Game-theoretic modeling of traffic in unsignalized intersection
  network for autonomous vehicle control verification and validation.
\newblock {\em IEEE Transactions on Intelligent Transportation Systems}, 2020.

\bibitem{vasirani2012market}
Matteo Vasirani and Sascha Ossowski.
\newblock A market-inspired approach for intersection management in urban road
  traffic networks.
\newblock {\em Journal of Artificial Intelligence Research}, 43:621--659, 2012.

\bibitem{wang2020game}
Mingyu Wang, Negar Mehr, Adrien Gaidon, and Mac Schwager.
\newblock Game-theoretic planning for risk-aware interactive agents.
\newblock {\em IROS}, 2020.

\end{thebibliography}
